\documentclass[]{article}

\pdfoutput=1

\usepackage[utf8]{inputenc}
\usepackage{amsmath}
\usepackage{amssymb}
\usepackage{MnSymbol}
\usepackage{graphicx}
\usepackage{subcaption}
\usepackage{fancyhdr}
\usepackage{enumitem}
\usepackage[nottoc,numbib]{tocbibind}
\usepackage{xcolor}
\usepackage{aliascnt}
\usepackage{hyperref}
\usepackage{cleveref}
\usepackage{mathtools}
\usepackage{amsfonts}
\usepackage{etoolbox}
\usepackage{amsthm}
\usepackage{pifont}

\newcommand{\setseparator}{\,:\,}
\newcommand{\converges}[2][\infty]{\underset{{#2}\rightarrow{#1}}{\longrightarrow}}
\newcommand{\differential}[1]{\,\mathrm{d}{#1}}
\newcommand{\semiinner}[3][\banachspace]{\left[{#2},{#3}\right]_{#1}}
\newcommand{\inner}[3][\hilbertspace]{\left<{#2},{#3}\right>_{#1}}
\renewcommand{\Re}{\operatorname{Re}}
\DeclareMathOperator{\vecspan}{span}

\DeclareMathOperator{\codim}{codim}
\DeclarePairedDelimiter\norm{\lVert}{\rVert}
\DeclarePairedDelimiter\modulus{\lvert}{\rvert}
\newcommand{\N}{\mathbb{N}}
\newcommand{\R}{\mathbb{R}}
\newcommand{\C}{\mathbb{C}}
\newcommand{\hilbertspace}{\mathcal{H}}
\newcommand{\banachspace}{\mathcal{B}}

\newcommand{\thetitle}{When is there a Representer Theorem?\\Nondifferentiable Regularisers and Banach spaces}
\newcommand{\theauthor}{Kevin Schlegel}
\newcommand{\theaddress}{Mathematical Institute\\
University of Oxford\\
Andrew Wiles Building, Radcliffe Observatory Quarter\\
Woodstock Road, Oxford, OX2 6GG, UK}
\newcommand{\theemail}{schlegel@maths.ox.ac.uk}

\makeatletter
\newtheoremstyle{theorem}
  {5mm}
  {7mm}
  {\addtolength{\@totalleftmargin}{7mm}
   \addtolength{\linewidth}{-7mm}
   \parshape 1 7mm \linewidth}
  {-7mm}
  {\bfseries}
  {}
  {\newline}
  {\normalfont\textbf{\thmname{#1}\thmnumber{ #2}}\textit{\thmnote{ (#3)}}}

\makeatother
\theoremstyle{theorem}

\newtheorem{theorem}{Theorem}[section]
\AfterEndEnvironment{theorem}{\noindent\ignorespaces}
\crefname{theorem}{theorem}{theorems}
\newaliascnt{proposition}{theorem}
\newtheorem{proposition}[proposition]{Proposition}
\aliascntresetthe{proposition}
\AfterEndEnvironment{proposition}{\noindent\ignorespaces}
\crefname{proposition}{proposition}{propositionss}
\newaliascnt{lemma}{theorem}
\newtheorem{lemma}[lemma]{Lemma}
\aliascntresetthe{lemma}
\AfterEndEnvironment{lemma}{\noindent\ignorespaces}
\crefname{lemma}{lemma}{lemmas}
\newaliascnt{remark}{theorem}
\newtheorem{remark}[remark]{Remark}
\aliascntresetthe{remark}
\AfterEndEnvironment{remark}{\noindent\ignorespaces}
\crefname{remark}{remark}{remarks}
\newaliascnt{corollary}{theorem}

\aliascntresetthe{corollary}
\AfterEndEnvironment{corollary}{\noindent\ignorespaces}
\crefname{corollary}{corollary}{corollaries}
\newaliascnt{definition}{theorem}
\newtheorem{definition}[definition]{Definition}
\aliascntresetthe{definition}
\AfterEndEnvironment{definition}{\noindent\ignorespaces}
\crefname{definition}{definition}{definitionss}
\newaliascnt{example}{theorem}

\aliascntresetthe{example}
\AfterEndEnvironment{example}{\noindent\ignorespaces}
\crefname{example}{example}{examples}

\newcommand{\proofend}{\hspace*{\fill}\ding{113}\\}

\let\oldendproof\endproof
\renewcommand{\endproof}{\proofend\oldendproof}

\makeatletter
\newtheoremstyle{notation}
  {5mm}
  {5mm}
  {\addtolength{\@totalleftmargin}{7mm}
   \addtolength{\linewidth}{-7mm}
   \parshape 1 7mm \linewidth}
  {-7mm}
  {\bfseries}
  {}
  {\newline}
  {\normalfont\textbf{\thmname{#1}\thmnumber{ #2}:}\thmnote{ (#3)}}

  \makeatother
\theoremstyle{notation}

\crefname{notation}{notation}{notations}

\newtheoremstyle{subproof}
  {0mm}
  {5mm}
  {}
  {}
  {\bfseries}
  {}
  {\newline}
  {\normalfont\textit{\textbf{\thmname{#1}\thmnumber{ #2}:}\thmnote{ (#3)}}}

\theoremstyle{subproof}

\newtheorem{subproof}{Part}
\AfterEndEnvironment{subproof}{\noindent\ignorespaces}
\crefname{subproof}{part}{parts}

\hypersetup{
  colorlinks=false,
  linkbordercolor=red,
  pdfborderstyle={/S/U/W .5}
}

\fancypagestyle{paper}{
  \fancyhead{}
  \fancyhead[L]{\footnotesize\nouppercase{\leftmark}}
  \fancyhead[R]{\footnotesize\theauthor}
  \fancyfoot{}
  \fancyfoot[C]{\thepage}
}

\renewcommand{\maketitle}{
  \thispagestyle{empty}
  \begin{center}
    {\Large \thetitle \par}
    \vspace{5mm}
    {\large \theauthor \par}
    \vspace{2mm}
    {\small \theaddress \par}
    {\small Email: \textit{\theemail} \par}
    \vspace{3mm}
    \today
    \vspace{8mm}
  \end{center}
}
\begin{document}

\maketitle

\pagestyle{paper}
{\Large\bf Abstract}\\
We consider a general regularised interpolation problem for learning a parameter vector from data. The well known representer theorem says that under certain conditions on the regulariser there exists a solution in the linear span of the data points. This is the core of kernel methods in machine learning as it makes the problem computationally tractable. Necessary and sufficient conditions for differentiable regularisers on Hilbert spaces to admit a representer theorem have been proved. We extend those results to nondifferentiable regularisers on uniformly convex and uniformly smooth Banach spaces. This gives a (more) complete answer to the question when there is a representer theorem. We then note that for regularised interpolation in fact the solution is determined by the function space alone and independent of the regulariser, making the extension to Banach spaces even more valuable.


\section{Introduction}
Regularisation is often described as a process of adding additional information or using previous knowledge about the solution to solve an ill-posed problem or to prevent an algorithm from overfitting to the given data. This makes it a very important method for learning a function from empirical data from very large classes of functions. Intuitively its purpose is to pick from all the functions that may explain the data the function which is the simplest in some suitable sense. Hence regularisation appears in various disciplines wherever empirical data is produced and has to be explained by a function. This has motivated to study regularisation problems in mathematics, statistics and computer science and in particular in machine learning theory (Cucker and Smale~\cite{smale2001}, Shawe-Taylor and Cristianini~\cite{shawe-taylor2004}, Micchelli and Pontil~\cite{micchelli2005}).\\
\ \\
In particular regularisation in Hilbert spaces has been studied in the literature for various reasons. First of all the existence of inner products allows for the design of algorithms with very clear geometric intuitions often based on orthogonal projections or the fact that the inner product can be seen as a kind of similarity measure.\\
But in fact crucial for the success of regularisation methods in Hilbert spaces is the well known \textit{representer theorem} which states that for certain regularisers there is always a solution in the linear span of the data points (Kimeldorf and Wahba~\cite{kimeldorf1971}, Cox and O'Sullivan~\cite{cox1990}, Sch\"{o}lkopf and Smola~\cite{smola1998,smola2001}). This means that the problem reduces to finding a function in a finite dimensional subspace of the original function space which is often infinite dimensional. It is this dimension reduction that makes the problem computationally tractable.\\
Another reason for Hilbert space regularisation finding a variety of applications is the \textit{kernel trick} which allows for any algorithm which is formulated in terms of inner products to be modified to yield a new algorithm based on a different symmetric, positive semidefinite kernel leading to learning in \textit{reproducing kernel Hilbert spaces} (Sch\"{o}lkopf and Smola~\cite{smola2002}, Shawe-Taylor and Cristianini~\cite{shawe-taylor2004}). This way nonlinearities can be introduced in the otherwise linear setup. Furthermore kernels can be defined on input sets which a priori do not have a mathematical structure by embeddings into a Hilbert space.\\
\ \\
When we are speaking of regularisation we are referring to \textit{Tikhonov regularisation}, i.e.~an optimisation problem of the form
$$\min\left\{\mathcal{E}({(\inner{f}{x_i},y_i)}^m_{i=1}) + \lambda\Omega(f)\setseparator f\in\hilbertspace\right\}$$
where $\hilbertspace$ is a Hilbert space, $\left\{(x_i,y_i) \setseparator i\in\N_m\right\}\subset\hilbertspace\times Y$ is a set of given input/output data with $Y\subseteq\R$, $\mathcal{E}\colon\R^m\times Y^m\rightarrow\R$ is an \textit{error function}, $\Omega\,\colon\hilbertspace\rightarrow\R$ a \textit{regulariser} and $\lambda>0$ is a{\textit{regularisation parameter}. Argyriou, Micchelli and Pontil~\cite{argyriou2009} show that under very mild conditions this regularisation problem admits a linear representer theorem if and only if the regularised interpolation problem
\begin{equation}
\label{eq:regularised_interpolation_hilbert}
  \min\left\{\Omega(f)\setseparator f\in\hilbertspace, \inner{f}{x_i}=y_i\,\forall i=1,\ldots,m\right\}
\end{equation}
admits a linear representer theorem. They argue that we can thus focus on the regularised interpolation problem which is more convenient to study. It is easy to see that their argument holds for the more general setting of the problem which we are going to introduce in this paper so we are going to take the same viewpoint in this paper and consider regularised interpolation.\\
\ \\
We will be interested in regularisation not only in Hilbert spaces as stated above but \textit{extend the theory to uniformly convex, uniformly smooth Banach spaces}, allowing for learning in a much larger variety of spaces. While any two Hilbert spaces of the same dimension are linearly isometrically isomorphic this is far from true for Banach spaces so they exhibit much richer geometric variety which may be exploited in learning algorithms. Furthermore we may encounter applications where the data has some intrinsic structure so that it cannot be embedded into a Hilbert space. Having a large amount of Banach spaces for potential embeddings may help to overcome this problem. Analogous to learning in reproducing kernel Hilbert spaces the generalisation to Banach spaces allows for learning in \textit{reproducing kernel Banach spaces} which have been introduced by Zhang, Xu and Zhang~\cite{zhang2009}. Our results regarding the existence of representer theorems are in line with Zhang and Zhang's work on representer theorems for reproducing kernel Banach spaces~\cite{zhang2012}.\\
But as we will show at the end of this paper the variety of spaces to pose the problem in is of even greater importance. It is often said that the regulariser favours solutions with a certain desirable property. We will show that in fact for regularised interpolation when we rely on the linear representer theorem it is essentially \textit{the choice of the space}, and only the choice of the space not the choice of the regulariser, which \textit{determines the solution}.\\
\ \\
It is well known that non-decreasing functions of the Hilbert space norm admit a linear representer theorem. Argyriou, Micchelli and Pontil~\cite{argyriou2009} showed that this condition is not just necessary but for differentiable regularisers also sufficient. In this paper we {\itshape remove the differentiablity condition\/} and show that any regulariser on a uniformly convex and uniformly smooth Banach space that admits a linear representer theorem is in fact very close to being radially symmetric, thus giving a (more) complete answer to the question when there is a representer theorem. Before presenting those results we present the necessary theory of semi-inner products to generalise the Hilbert space setting considered by Argyriou, Micchelli and Pontil to Banach spaces.\\
\ \\
In \cref{sec:semi-inner-product} we will introduce the notion of \textit{semi-inner products} as defined by Lumer~\cite{lumer1961} and later extended by Giles~\cite{giles1967}. We will state the results without proofs as they mostly are not difficult and can be found in the original papers. Another extensive reference about semi-inner products and their properties is the work by Dragomir~\cite{dragomir2004}.\\
After introducing the relevant theory we will present the generalised regularised interpolation problem in \cref{sec:representer-theorem}, replacing the inner product in \cref{eq:regularised_interpolation_hilbert} by a semi-inner product. We then state one of the main results of the paper that regularisers that admit a representer theorem are almost radially symmetric in a way that will be made precise in the statement. Before giving the proof of the theorem we state and prove two essential lemmas capturing most of the important structure of the problem to prove the theorem. We finish the section by giving the proof of the main result.\\
Finally in~\cref{sec:space-dependence} we prove that in fact for admissible regularisers there is a unique solution of the regularised interpolation problem in the linear span of the data and it is independent of the regulariser. This in particular means that we may choose the regulariser which is most suitable for our task at hand without changing the solution.

\subsection{Notation}
Before the main sections we briefly introduce some notation used throughout the paper. We use $\N_m$ as a shorthand notation for the set $\{1,\ldots,m\}\subset\N$. We will assume we have $m$ data points $\left\{(x_i,y_i) \setseparator i\in\N_m\right\}\subset\banachspace\times Y$, where $\banachspace$ will always denote a uniformly convex, uniformly smooth real Banach space and $Y\subseteq\mathbb{R}$. Typical examples of $Y$ are finite sets of integers for classification problems, e.g. $\{-1,1\}$ for binary classification, or the whole of $\R$ for regression.\\
We briefly recall the definitions of a Banach space being uniformly convex and uniformly smooth, further details can be found in~\cite{brezis2011, lindenstrauss1979, koethe1983}.
\begin{definition}[Uniformly convex Banach space]
  A normed vector space $V$ is said to be uniformly convex if for every $\varepsilon>0$ there exists a $\delta>0$ such that if $x,y\in V$ with $\norm{x}_V=\norm{y}_V=1$ and $\norm{x-y}_V > \varepsilon$ then $\norm{\frac{x+y}{2}}_V < 1 - \delta$.
\end{definition}
\begin{definition}[Uniformly smooth Banach space]
  A normed vector space $V$ is said to be uniformly smooth if for every\\
  $\varepsilon>0$ there exists $\delta>0$ such that if $x,y\in V$ with $\norm{x}_V=1, \norm{y}_V\leq\delta$ then $\norm{x+y}_V+\norm{x-y}_V\leq 2+\varepsilon\norm{y}_V$.
\end{definition}
\begin{remark}
\label{rmk:uniform_smoothness}
  There are two equivalent conditions of uniform smoothness which we will make use of in this paper.
  \begin{enumerate}[label=(\roman*)]
  \item The modulus of smoothness of the space $V$ is defined as
    \begin{equation}
      \label{eq:modulus_of_smoothness}
      \rho_V(\delta) = \sup\left\{\frac{\norm{x+y}_V+\norm{x-y}_V}{2} - 1 \setseparator \norm{x}_V=1, \norm{y}_V=\delta\right\}
    \end{equation}
    Now $V$ is uniformly smooth if and only if $\frac{\rho_V(\delta)}{\delta}\converges[0]{\delta}0$.
  \item The norm on $V$ is said to be uniformly Fr\'{e}chet differentiable if the limit
    $$\lim\limits_{t\rightarrow 0}\frac{\norm{x+t\cdot y}_V-\norm{x}_V}{t}$$
    exists uniformly for all real $t$ and $x,y\in V$ with $\norm{x}_V=\norm{y}_V= 1$. The space $V$ is uniformly smooth if its norm is uniformly Fr\'{e}chet differentiable.
  \end{enumerate}
\end{remark}
We always write $\hilbertspace$ to denote a Hilbert space and for the first part of \cref{sec:semi-inner-product} we will be speaking of general normed linear spaces denoted by $V$. Once we have seen the reasons to require the space to be a uniformly convex and uniformly smooth Banach space the remainder of \cref{sec:semi-inner-product} and the paper will consider such spaces denoted by $\banachspace$. When only the norm $\norm{\cdot}_\banachspace$ on $\banachspace$ is considered the subscript will often be omitted for simplicity. Throughout we will denote the inner product on a Hilbert space by $\inner{\cdot}{\cdot}$ and a semi-inner product on a normed linear space by $\semiinner[V]{\cdot}{\cdot}$.

\section{Semi-inner product spaces}
\label{sec:semi-inner-product}
There are various definitions of semi-inner products aiming to generalise Hilbert space methods to more general cases. The notion of semi-inner products we are going to use was first introduced by Lumer~\cite{lumer1961} and further developed by Giles~\cite{giles1967}. In comparison to inner products the assumption of (conjugate) symmetry, or equivalently additivity in the second argument, is dropped. This means that we need to assume the Cauchy-Schwarz inequality to make sure that it holds as it is crucial for the semi-inner products to have inner-product like behaviour. In the original definition Lumer did not assume homogeneity in the second argument but Giles argued that one can assume it without any significant restrictions. We will hence be including homogeneity in our assumptions.\\
An extensive overview of the theory of this and other notions of semi-inner products can be found in Dragomir~\cite{dragomir2004}.\\
In this section only we state all results for real or complex vector spaces as all of them are valid for the complex case. Throught this section we will thus denote the field by $\mathbb{F}$. In the subsequent sections where we present the main contributions of this paper we will return to real vector spaces as it is at this point not clear whether the results remain valid for complex vector spaces.

\begin{definition}[Semi-inner product]
 A semi-inner product (s.i.p.) on a real or complex vector space $V$ is a map $\semiinner[V]{\cdot}{\cdot}:V\times V\rightarrow\mathbb{F}$ with the following properties:
 \begin{enumerate}[label=(\roman*),series=semiinner]
  \item\label{itm:linearity} Linearity in the first argument:\\
    $\semiinner[V]{\lambda x+\mu y}{z} = \lambda\semiinner[V]{x}{z}+\mu\semiinner[V]{y}{z} \quad$ for all $x,y,z\in V$ and $\lambda,\mu\in\mathbb{F}$
  \item\label{itm:posdef} Positive definiteness:\\
    $\semiinner[V]{x}{x}\geq 0$ and $\semiinner[V]{x}{x}=0 \Leftrightarrow x=0$
  \item\label{itm:CSineq} Cauchy-Schwarz inequality:\\
    $\modulus{\semiinner[V]{x}{y}}^2 \leq \semiinner[V]{x}{x}\semiinner[V]{y}{y}$
  \item\label{itm:homogeneity} (Conjugate) homogeneity in the second argument:\\
    $\semiinner[V]{x}{\lambda y} = \overline{\lambda}\semiinner[V]{x}{y} \quad$ for all $x,y\in V$ and $\lambda\in\mathbb{F}$
 \end{enumerate}
\end{definition}
With these properties a semi-inner product $\semiinner[V]{\cdot}{\cdot}$ induces a norm $\semiinner[V]{x}{x}=\norm{x}_V$ on $V$. Conversely every norm $\norm{\cdot}_V$ on a linear space $V$ is induced by at least one semi-inner product, i.e.\@ there exists at least one semi-inner product $\semiinner[V]{\cdot}{\cdot}$ such that $\norm{x}_V=\semiinner[V]{x}{x}$. This means that every normed linear space is a s.i.p.\@ space. Consequently we say that an s.i.p.\@ space $V$ is uniformly convex if the norm induced by $\semiinner[V]{\cdot}{\cdot}$ is uniformly convex and the s.i.p.\@ space is uniformly smooth if the induced norm is uniformly smooth.\\
The semi-inner product inducing the norm is not unique in general though. It turns out that we have uniqueness if the norm is differentiable which is closely linked to a weak continuity property in the second argument of the inducing semi-inner product.

\begin{proposition}
  If the norm $\norm{\cdot}_V$ on $V$ is uniformly Fr\'{e}chet differentiable as defined in \cref{rmk:uniform_smoothness}, then
  \begin{equation}
    \label{eq:weak_continuity_property}
    \Re\semiinner[V]{x}{y+tx} \rightarrow \Re\semiinner[V]{x}{y}
  \end{equation}
  uniformly for every $x,y\in V$ with $\norm{x}_V=\norm{y}_V=1$ as  $\R\ni t\rightarrow 0$. Furthermore the differential of the norm for $x\neq0$ is given by
  $$\lim\limits_{t\rightarrow 0} \frac{\norm{x+ty}_V - \norm{x}_V}{t} = \frac{\Re\semiinner[V]{y}{x}}{\norm{x}_V}$$
  This in particular means that the semi-inner product inducing a uniformly Fr\'{e}chet differentiable norm is unique.
\end{proposition}

The existence of a semi-inner product allows us to define a notion of orthogonality analogous to orthogonality in Hilbert spaces by requiring the semi-inner product to be zero. The lack of symmetry of the semi-inner product thus means that our notion of orthogonality is not symmetric in general and $x$ normal to $y$ does not imply that $y$ is normal to $x$.

\begin{definition}[Orthogonality]
 Let $V$ be a s.i.p.\ space. For $x,y\in V$ we say $x$ is normal to $y$ if $\semiinner[V]{y}{x}=0$.\\
 A vector $x\in V$ is normal to a subspace $U\subset V$ if $x$ is normal to all $y\in U$.
\end{definition}

Various generalisations of orthogonality have been developed which are equivalent conditions to the inner product being zero in a Hilbert space but generalise to normed linear spaces. One of these notions of orthogonality is James orthogonality~\cite{james1947}. The equivalence of James orthogonality with the inner product being zero in a Hilbert space generalises to smooth Banach spaces in which James orthogonality is equivalent to the unique semi-inner product being zero. James states that his definition is closely related to linear functionals and hyperplanes which is essential for our applications as we will see in the main part of the paper.

\begin{proposition}[James orthogonality]
  \label{prop:james_orthogonality}
  In a uniformly smooth s.i.p.\@ space  semi-inner product orthogonality is equivalent to James orthogonality, namely for $x,y\in V$
  $$\semiinner[V]{y}{x} = 0 \Leftrightarrow \norm{x+\lambda y}_V \geq \norm{x}_V \quad \mbox{ for all } \lambda\in\mathbb{F}$$
\end{proposition}

This relation to James orthogonality also helps to get a geometric understanding of what orthogonality means in a s.i.p.~space. From \cref{prop:james_orthogonality} it is immediately clear that $x$ being normal to $y$ means that the vector $y$ is tangent to the ball $B(0, \norm{x})$ at the point $x$, where $B(0,\norm{x})$ is the ball of radius $\norm{x}$ centred at the origin.\\
Having defined what it means to be orthogonal to a linear subspace we can also define the orthogonal complement of a subspace. It will become clear later that this definition coincides with the usual definition of orthogonal complements in Banach spaces via the dual space.

\begin{definition}[Orthogonal Complement]
  Let $V$ be a s.i.p.~space and $U$ a closed linear subspace. Then the orthogonal complement of $U$ is defined to be
  $$U^\perp = \left\{x_\perp\in V \setseparator \semiinner[V]{x}{x_\perp}=0\,\forall x\in U\right\}$$
\end{definition}

If the space is a uniformly convex Banach Space it is not difficult to see that there is a unique orthogonal decomposition for every $x\in V$. This is because it is known that in a uniformly convex space there is a unique closest point in a closed linear subspace and one easily checks that this immediately leads to a unique orthogonal decomposition.

\begin{proposition}[Orthogonal Decomposition]
  Let $V$ be a uniformly convex s.i.p.~space. Then for any closed linear subspace $U\subset V$ there exists a unique orthogonal decomposition, more precisely for any $x\in V$ there exists a unique $x_0\in U$ and a unique $x_\perp\in U^\perp$ such that $x=x_0+x_\perp$.
\end{proposition}

Under these assumptions
we are also able to establish a Riesz representation theorem using the semi-inner product.

\begin{theorem}[Riesz representation theorem]
 Let $V$ be a uniformly convex, uniformly smooth s.i.p.\ space. Then for every $f\in V^\ast$, the continuous dual space of $V$, there exists a unique vector $y\in V$ such that
 $$f(x)=\semiinner[V]{x}{y} \quad \mbox{ for all } x\in V$$
 Furthermore
 $$\norm{y}_V = \norm{f}_{V^\ast}$$
\end{theorem}

This theorem is crucial for the development of the theory in this paper as it means that the duality map $x\mapsto x^\ast$ given by
$$x^\ast(y) = \semiinner[V]{y}{x}\quad\forall y\in V$$
is an isometric isomorphism from $V$ to $V^\ast$. It is essential to note that this map is linear if and only if $V$ is a Hilbert space.\\
\ \\
Summarizing the above results we see that a necessary structure to have a unique semi-inner product inducing the norm and allowing for a Riesz representation theorem is that the space is a uniformly convex and uniformly Fr\'{e}chet differentiable Banach space. For simplicity we will be calling such spaces uniform.

\begin{definition}[Uniform Banach space]
 We say a space $V$ is uniform if it is a uniformly convex and uniformly Fr\'{e}chet differentiable Banach space.\\
\end{definition}

For the remainder of the paper we will only be working with uniform Banach spaces and throughout denote them by $\banachspace$.\\
Note that any Banach space that is uniformly convex or uniformly Fr\'{e}chet differentiable is reflexive. Further a Banach space is uniformly Fr\'{e}chet differentiable if and only if its dual space is uniformly convex. Thus for a uniform Banach space $\banachspace$ its dual space $\banachspace^\ast$ is also uniform and its norm-inducing semi-inner product is given by
$$\semiinner[\banachspace^\ast]{x^\ast}{y^\ast} = \semiinner[\banachspace]{y}{x}$$
We already know that the duality map is a homogeneous isometric isomorphism. Lastly we note that in fact it is also norm-to-norm continuous.The proof for this is standard and can be found in the appendix.
\begin{proposition}
  \label{prop:dual_map_cts}
  The duality map $\ast:\banachspace\rightarrow\banachspace^\ast, x\mapsto x^\ast$ is norm-to-norm continuous.\\
  In particular this shows that in fact \cref{eq:weak_continuity_property} can be strengthened to
  $$\semiinner{z}{x+ty} \rightarrow \semiinner{z}{x}$$
  for all $x,y,z\in\banachspace$ and $t\in\C$.\\
\end{proposition}

Thus the dual map is a homeomorphism from $\banachspace$ to $\banachspace^\ast$ with the norm topologies.

\section{Existence of Representer Theorems}
\label{sec:representer-theorem}
The definitions and results of the previous section allow us to consider the regularised interpolation problem
\begin{equation}
  \label{eq:regularised_interpolation_semiinner}
  \min\left\{\Omega(f)\setseparator f\in\banachspace,\semiinner[\banachspace]{f}{x_i} = y_i\,\forall i \in\N_m\right\}
\end{equation}
where the domain $\banachspace$ of the interpolation problem is a real uniform Banach space. This generalises the setting considered by Argyriou, Micchelli and Pontil in~\cite{argyriou2009} where the case of a Hilbert space domain is considered. In that setting the linear representer theorem states that there exists a solution to the interpolation problem which is in the linear span of the data points. Our work, similarly as~\cite{miccelli2004}, hints that in its essence the representer theorem is a result about the dual space rather than the space itself. Since in a Hilbert space the dual element is the element itself this doesn't become apparent in this setting and we obtain a result in the space itself. As the duality map is nonlinear for any Banach space which is not Hilbert we need to adjust the formulation of the representer theorem. Namely the linear representer theorem in a uniform Banach space states that there exists a solution such that its dual element is in the linear span of the dual elements of the data points. This is made precise in the following definition which is the analogue of Argyriou, Micchelli and Pontil calling regularisers which always admit a linear representer theorem admissible.

\begin{definition}[Admissible Regulariser]
  We say a function $\Omega:\banachspace\rightarrow\R$ is admissible if for any $m\in\N$ and any given data $\left\{(x_i,y_i) \setseparator i\in\N_m\right\}\subset \banachspace\times Y$ such that the interpolation constraints can be satisfied the regularised interpolation problem \cref{eq:regularised_interpolation_semiinner} admits a solution $f_0$ such that its dual element is of the form
  $$f^\ast_0=\sum\limits_{i=1}^{m} c_{i}x^\ast_i$$
\end{definition}

With this definition at hand it is now our goal to classify all admissible regularisers. It is well known that being a non-decreasing function of the norm on a Hilbert space is a sufficient condition for the regulariser to be admissible. By a Hahn-Banach argument similar as e.g.\ in Zhang, Zhang~\cite{zhang2012} this generalises to our case of uniform Banach spaces. Below we show that this condition is already almost necessary in the sense that admissible regularisers cannot be very far from being radially symmetric.

\begin{theorem}
\label{thm:admissible_regulariser}
  A function $\Omega$ is admissible if and only if it is of the form
  $$\Omega(f) = h(\semiinner[\banachspace]{f}{f})$$
  for some non-decreasing $h$ whenever $\norm{f} \neq r$ for $r\in\mathcal{R}$. Here $\mathcal{R}$ is an at most countable set of radii where $h$ has a jump discontinuity. For any $f$ with $\norm{f}=r\in\mathcal{R}$ the value $\Omega(f)$ is only constrained by the monotonicity property, i.e.\ it has to lie in between $\lim\limits_{t\nearrow r}h(t)$ and $\lim\limits_{t\searrow r}h(t)$.\\
  \ \\
  In other words, $\Omega$ is radially non-decreasing and radially symmetric except for at most countably many circular jump discontinuities. In those discontinuities the function value is only limited by its monotonicity property.
\end{theorem}

In~\cite{argyriou2009} Argyriou, Micchelli and Pontil show that any admissible regulariser on a Hilbert space is non-decreasing in orthogonal directions. An analogous result is true for uniform Banach spaces but with orthogonality not being symmetric and our intuition gained from the equivalence with James orthogonality we see that in fact it is tangential directions in which the regulariser is non-decreasing. This also becomes clear from the proves in~\cite{argyriou2009}, in particular when proving radial symmetry.\\
Before we can prove the analogous result for uniform Banach spaces we need to show that we can extend this tangential bound considerably and a function that is non-decreasing in tangential directions is in fact non-decreasing in norm as is made precise in the following \namecref{lma:circular_bound}.

\begin{lemma}
\label{lma:circular_bound}
  If $\Omega(f)\leq\Omega(f+f_T)$ for all $f,f_T\in\banachspace$ such that $\semiinner{f_T}{f}=0$ then for any fixed $\hat{f}$ we have that $\Omega(\hat{f})\leq\Omega(f)$ for all $f$ such that $\norm{\hat{f}}<\norm{f}$.
\end{lemma}

\begin{proof}
  \begin{subproof}[Bound $\Omega$ on the half space given by the tangent through $\hat{f}$]\label{spf:half_space_bound}
    We start by showing that $\Omega$ is radially non-decreasing. Since it is non-decreasing along tangential directions this immediately gives the claimed bound for the entire half space given by the tangent through $\hat{f}$. The idea of the proof is to move out along a tangent until we can move back along another tangent to hit a given point along the ray $\lambda\cdot\hat{f}$ as shown in \cref{fig:radially_increasing}.
    \begin{figure}[h]
      \centering
      \includegraphics[width=.65\linewidth]{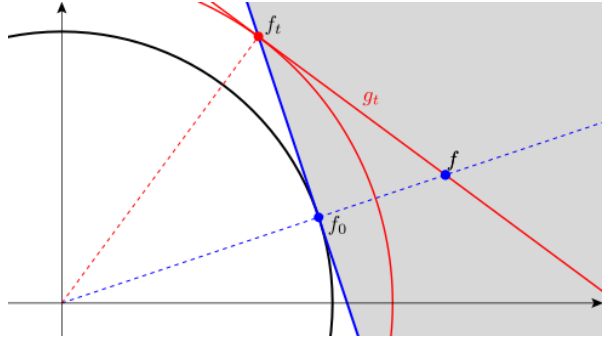}
      \caption{\footnotesize We can extend the tangential bound to the ray $\lambda\cdot f_0$ by finding the point $f_t$ along the tangent from where the tangent to $f_t$ hits the desired point on the ray. Via the tangents to points along the ray the bound then extends to the shaded half space.}
      \label{fig:radially_increasing}
    \end{figure}\\
    Fix some $\hat{f}\in\banachspace$ and $1<\lambda\in\R$ and set $f=\lambda\cdot \hat{f}$. We need to show that $\Omega(f)\geq\Omega(\hat{f})$. Let $f_T\in\banachspace$ be such that $\semiinner{f_T}{\hat{f}}=0$ or equivalently $\norm{\hat{f}+t\cdot f_T}>\norm{\hat{f}}$ for all $t\neq0$. Now let
    \begin{gather*}
      f_t=\hat{f}+t\cdot f_T\\
      g_t=f-f_t=(\lambda-1)\cdot \hat{f}-t\cdot f_T
    \end{gather*}
    so that $f_t+g_t=f$. Note that by strict convexity and continuity of the norm $\norm{f_t}=\norm{\hat{f}+t\cdot f_T}$ is continuous and strictly increasing in $t$.\\
    Now since $t\cdot f_T$ is the tangent through $\hat{f}$ and $g_t$ points from $f_t$ to $f$, for small $t$ for which $\norm{f_t}<\norm{f}$ we must have that
    \begin{equation}
      \label{eq:small_t}
      \norm{f_t+s\cdot g_t}>\norm{f_t} \mbox{ for all } s\in(0,1)
    \end{equation}
    On the other hand for $t$ big enough so that $\norm{f_t}>\norm{f}$ we thus must have
    \begin{equation}
      \label{eq:large_t}
      \norm{f_t+s\cdot g_t}<\norm{f_t} \mbox{ for } s \mbox{ small enough}
    \end{equation}
    But we know that
    $$\lim\limits_{s\rightarrow0}\frac{\norm{f_t+s\cdot g_t}-\norm{f_t}}{s}=\frac{\semiinner{g_t}{f_t}}{\norm{f_t}}=\frac{f^\ast_t(g_t)}{\norm{f_t}}$$
    and since the dual map is norm-to-norm continuous $\frac{f^\ast_t(g_t)}{\norm{f_t}}$ is clearly continuous in $t$. By above discussion the expression is positive for small $t$ and negative for large $t$ so by the intermediate value theorem there exists $t_0$ such that
    $$\frac{f^\ast_{t_0}(g_{t_0})}{\norm{f_{t_0}}} = \frac{\semiinner{g_{t_0}}{f_{t_0}}}{\norm{f_{t_0}}}= 0$$
    so that indeed $\semiinner{g_{t_0}}{f_{t_0}} = 0$ and thus $g_{t_0}$ is tangential to $f_{t_0}$. But this means that $\Omega(f)\geq\Omega(f_{t_0})\geq\Omega(\hat{f})$ as claimed.\\
    Hence we have the bound along the entire ray $\lambda\cdot \hat{f}$ for $1<\lambda\in\R$  which extends along all tangents through those points to the half space given by the tangent through $\hat{f}$, i.e. the shaded region in \cref{fig:radially_increasing}.
  \end{subproof}

  \begin{subproof}[Extend the bound around the circle]
    Next we note that we can actually extend the bound further to apply all the way around the circle, namely $\Omega(f)\geq\Omega(\hat{f})$ for all $f$ such that $\norm{f}>\norm{\hat{f}}$. This is done by considering $f_t=\hat{f}+t\cdot f_T$ as before but then instead of following the tangent into the half space just considered we follow the tangent in the opposite direction around the circle, as shown in \cref{fig:radiating_bound}. We fix another point along that tangent and repeat the process, moving around the circle. We claim that by making the step size along each tangent small enough we can this way move around the circle while staying arbitrarily close to it.\\
    More precisely we need to show that the distance a step along a tangent takes us away from the circle decreases faster than the step along the tangent so that we move considerably further around the circle than away from it with each step, as shown in \cref{fig:step_size}.\\
    \begin{figure}
      \begin{subfigure}{.5\textwidth}
        \centering
        \includegraphics[width=.95\linewidth]{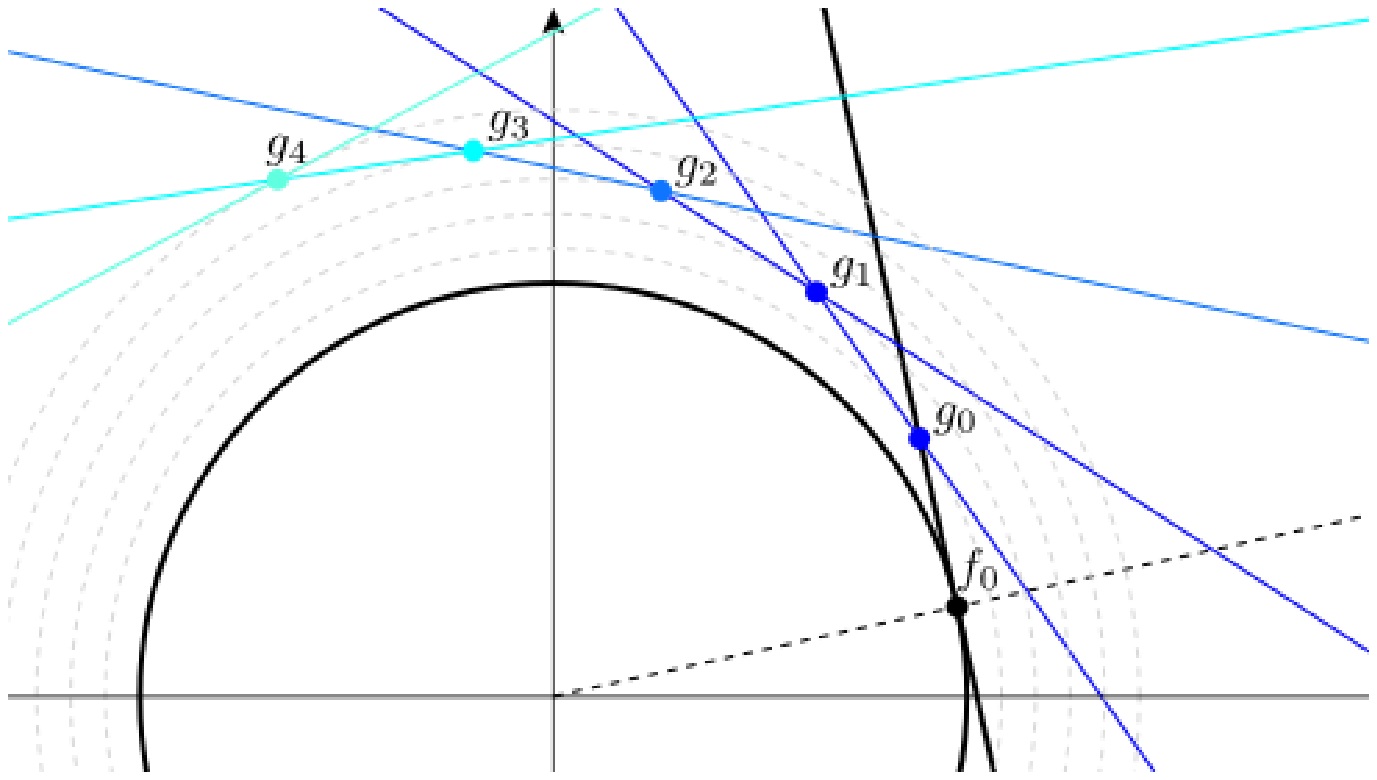}
        \caption{\footnotesize By repeatedly taking steps along tangents we can move all the way around the circle.}
        \label{fig:radiating_bound}
      \end{subfigure}
      \begin{subfigure}{.5\textwidth}
        \centering
        \includegraphics[width=.95\linewidth]{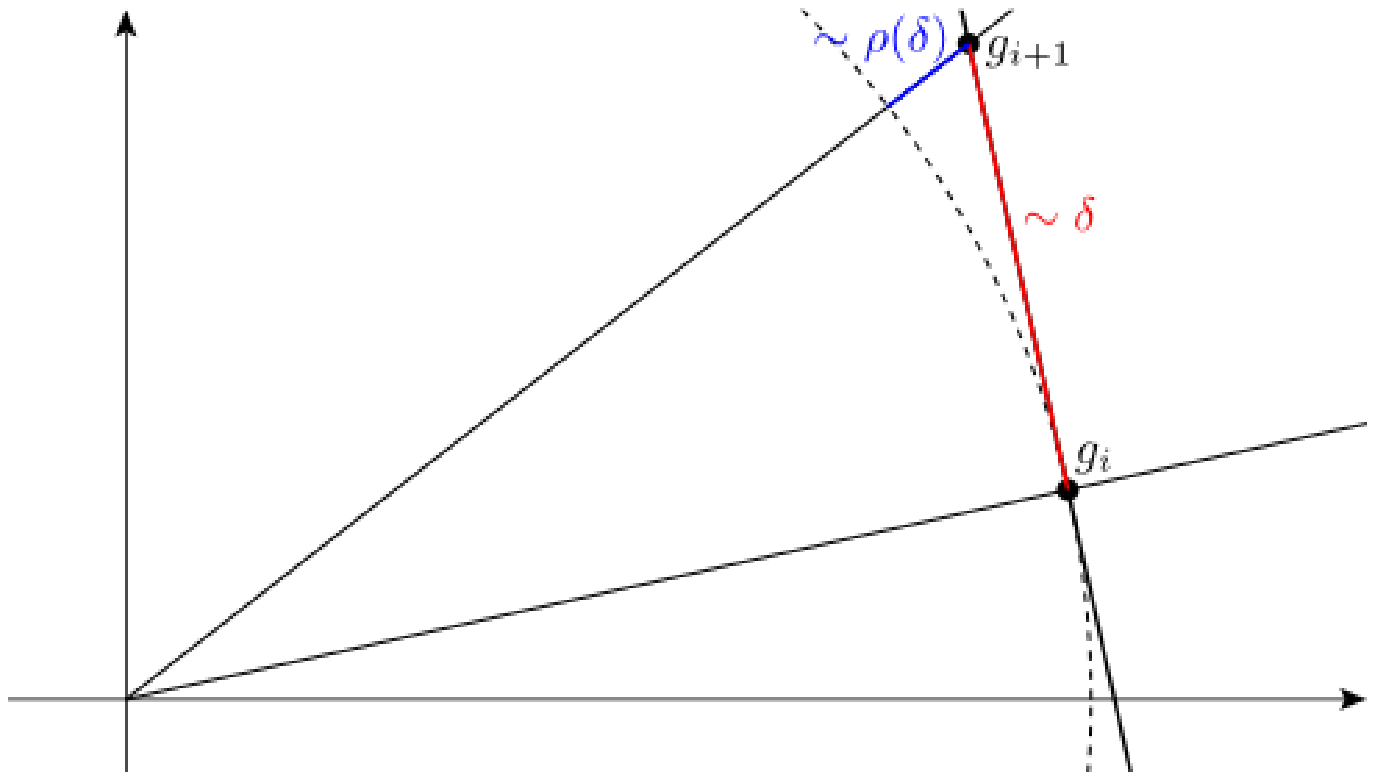}
        \caption{\footnotesize When decreasing the step size along a tangent the step size away from the circle decreases significantly faster so that by making the steps along tangents small enough we can reach any point arbitrarily close to the circle.}
        \label{fig:step_size}
      \end{subfigure}
    \end{figure}\\
    \noindent As stated in \cref{eq:modulus_of_smoothness} let
    $$\rho_\banachspace(\delta) = \sup\left\{\frac{\norm{f+g}+\norm{f-g}}{2}-1 \setseparator \norm{f}=1, \norm{g}=\delta \right\}$$
    be the modulus of smoothness of the space $\banachspace$. For $f,f_T\in\banachspace$ such that\\
    $\semiinner{f_T}{f}=0$, $\norm{f}=1$, $\norm{f_T}=\delta$ we have that $\norm{f+t\cdot f_T}>\norm{f}$ for all $t\neq0$ so in particular $\norm{f-f_T}>\norm{f}$. We thus easily see that
    \begin{align*}
      \norm{f+f_T} & \leq 2 + 2\rho_\banachspace(\delta) - \norm{f-f_T}\\
                 & < 2 + 2\rho_\banachspace(\delta) - \norm{f}\\
                 & = 1 + 2\rho_\banachspace(\delta)
    \end{align*}
    This means that for a step of order $\delta$ along a tangent, i.e.\ $f_T$ of length $\delta$, we take a step of order $\rho_\banachspace(\delta)$ away from the circle. But since $\banachspace$ is uniformly smooth we have that $\frac{\rho_\banachspace(\delta)}{\delta}\rightarrow0$ as $\delta\rightarrow 0$ proving that for small enough $\delta$ indeed the step away from the circle is significantly smaller than the step along the tangent as shown in \cref{fig:step_size}.\\
    Combining both arguments this proves that we can reach any point with norm greater than $\norm{\hat{f}}$ from $\hat{f}$ only by moving along tangents giving the claimed bound.
  \end{subproof}
\end{proof}

Having proved this lemma we are now in the position to prove that indeed any admissible regulariser on a uniform Banach space is non-decreasing in tangential directions. Note that the previous lemma will also play a crucial role in removing the differentiability assumption when establishing the closed form representation of the regulariser in \cref{thm:admissible_regulariser}.

\begin{lemma}
\label{lma:admissibility}
  A function $\Omega$ is admissible if and only if for every $f,f_T\in\banachspace$ such that $\semiinner{f_T}{f}=0$ we have
  $$\Omega(f)\leq\Omega(f + f_T)$$
  if and only if for any fixed $\hat{f}$ and all $f$ such that $\norm{\hat{f}} < \norm{f}$ we have
  $$\Omega(\hat{f})\leq\Omega(f)$$

\end{lemma}

\begin{proof}
  \begin{subproof}[$\Omega$ admissible $\Rightarrow$ nondecreasing along tangential directions]
    Fix any $f\in\banachspace$ and consider the regularised interpolation problem
    $$\min\left\{\Omega(g)\setseparator g\in\banachspace,\semiinner[\banachspace]{f}{g} = \semiinner[\banachspace]{f}{f}\right\}$$
    As $\Omega$ is assumed to be admissible there exists a solution with dual element in $\vecspan\{f^\ast\}$ which by homogeneity of the dual map clearly is $f$ itself. But if $f_T$ is such that $\semiinner{f_T}{f}=0$ then $\semiinner{f+f_T}{f}=\semiinner{f}{f}$ so $f+f_T$ also satisfies the constraints and hence necessarily $\Omega(f+f_T)\geq \Omega(f)$ as claimed. The second claim follows immediately from \cref{lma:circular_bound}.
  \end{subproof}

  \begin{subproof}[Nondecreasing along tangential directions $\Rightarrow$ $\Omega$ admissible]
    Conversely fix any data $\left\{(x_i,y_i)\setseparator i\in\N_m\right\}\subset\banachspace\times Y$ such that the interpolation constraints can be satisfied. Let $f_0$ be a solution to the regularised interpolation problem. If $f_0^\ast\in\vecspan\{x_i^\ast\}$ we are done so assume it is not. We let
    $$X^\ast=\vecspan\{x_i^\ast\}\subset\banachspace^\ast \qquad X=\{x\in\banachspace\setseparator x^\ast\in X^\ast\}$$
    Further denote by $Z\subset\banachspace$ the space corresponding to the orthogonal complement of $X^\ast$ i.e.
    $$Z = \{f_T\in\banachspace\setseparator f^\ast_T\in {(X^\ast)}^\perp\} = \{f_T\in\banachspace\setseparator\semiinner{f_T}{x_i}=0\,\forall i\in\N_m\}$$
    Thus $Z^\ast\cap X^\ast = \{0\}$ and by assumption $f_0^\ast\not\in X^\ast$ and so also $\vecspan\{f_0^\ast\}\cap X^\ast = \{0\}$.\\
    \ \\
    Now by definition we have that
    $$Z=\bigcap_{i\in\N_m}\ker(x_i^\ast)$$
    so the codimension of $Z$ is $m$. Without loss of generality we can assume that not all $y_i$ are zero as otherwise $f_0=f_0^\ast=0$ is a trivial solution in the span of the data points. Since not all $y_i$ are zero $f_0\not\in Z$ and thus $\codim(\vecspan\{f_0,Z\}) = m-1$. But since $X^\ast=\vecspan\{x_i^\ast\}$ and the dual map is a homeomorphism $X$ is homeomorphic to a linear space of dimension $m$. This means that that $X\cap\vecspan\{f_0,Z\}$ is homeomorphic to a one-dimensional space and hence in particular contains a nonzero element.\\
    Now fix such $0\neq f\in X\cap\vecspan\{f_0,Z\}$. As we noted earlier $f$ being nonzero means that $f\not\in\vecspan\{f_0\}$ and $f\not\in Z$. Thus $f=\lambda f_0 + \mu g$ for $\lambda,\mu\neq 0, g\in Z$. By homogeneity of the dual map $\lambda\cdot X=X$ and so
    $$f\in X\cap\vecspan\{f_0,Z\}\Leftrightarrow\frac{1}{\lambda}f\in X\cap\vecspan\{f_0,Z\}$$
    and thus
    \begin{equation}
      \frac{1}{\lambda}f=f_0+\frac{\mu}{\lambda}g = f_0+\widetilde{g}\in X\cap\vecspan\{f_0,Z\}
    \end{equation}
    with $\widetilde{g}=\frac{\mu}{\lambda}g\in Z$.\\
    This means we have constructed an $\overline{f_0}=f_0 + f_T$ with dual element in the span of the data points and $f_T\in Z$ which means by definition of $Z$ that $\overline{f_0}$ satisfies the interpolation constraints. It remains to show that in fact $\overline{f_0}$ is in norm at most as large as $f_0$.\\
    To this end note that for all $f_T\in Z$ by definition $\semiinner[\banachspace^\ast]{x^\ast}{f^\ast_T}=0$ for all $x^\ast\in X^\ast$ and hence we see that for $\overline{f_0}=f_0+f_T\in X$ we get that
    $$\semiinner[\banachspace^\ast]{{(f_0+f_T)}^\ast}{f_T^\ast} = \semiinner{f_T}{f_0+f_T} = 0$$
    But by the equivalence with James orthogonality this means that\\
    $\norm{(f_0+f_T) + t\cdot f_T}>\norm{f_0+f_T}$ for all $t\neq0$ or equivalently
    $$\norm{f_0+f_T} = \min\limits_{t\in\R}\norm{f_0+t\cdot f_T}$$
    In particular $\norm{\overline{f_0}}=\norm{f_0+f_T}<\norm{f_0+0\cdot f_T}=\norm{f_0}$.\\
    But by \cref{lma:circular_bound} we know that for a function which is non-decreasing along tangential directions is non-decreasing in norm so $\norm{\overline{f_0}}<\norm{f_0}$ implies that $\Omega(\overline{f_0})\leq\Omega(f_0)$ and so we have found a solution with dual element in the span of the data points as claimed.
  \end{subproof}
\end{proof}

Using those two results we can now give the proof that admissible regularisers are almost radially symmetric in the sense of \cref{thm:admissible_regulariser}.\\

\begin{proof}[Of \cref{thm:admissible_regulariser}]
  \begin{subproof}[$\Omega$ continuous in radial direction implies $\Omega$ radially symmetric]\label{sbpf:radial_continuity}
    We now show that instead of differentiability, the assumption that $\Omega$ is continuous in radial direction is sufficient to conclude that it has to be radially symmetric. We prove this by contradiction. Assume $\Omega$ is admissible but not radially symmetric. Then there exists a radius $r$ so that $\Omega$ is not constant on the circle with radius $r$ and hence there are two points $f$ and $g$ so that, without loss of generality,  $\Omega(f)>\Omega(g)$.\\
    But then by \cref{lma:circular_bound} for all $1<\lambda\in\R$ we have $\Omega(\lambda g)\geq\Omega(f)$ and thus as $\Omega$ non-negative and non-decreasing $\modulus{\Omega(\lambda g)-\Omega(g)}\geq\modulus{\Omega(f)-\Omega(g)}>0$ contradicting radial continuity of $\Omega$. Hence $\Omega$ has to be constant along every circle as claimed.\\
  \end{subproof}

  \begin{subproof}[Radial mollification preserves being nondecreasing in tangential directions]\label{sbpf:mollified_regulariser}
    The observation in \cref{sbpf:radial_continuity} is useful as we can easily radially mollify a given $\Omega$ so that the property of being non-decreasing along tangential directions is preserved.\\
    Indeed let $\rho$ be a mollifier such that $\rho:\R\rightarrow [0,\infty)$ with support in $[-1,0]$ and for each ray given by some $f_0\in\banachspace$ of unit norm, define the mollified regulariser by
    $$\widetilde{\Omega}(sf_0) = \int\limits_\R \rho(t)\Omega\left((s-t)f_0\right)\differential{t}$$
    We thus obtain a radially mollified regulariser on $\banachspace$ given by
    \begin{align*}
      \widetilde{\Omega}(f) = \widetilde{\Omega}\left(\norm{f}\frac{f}{\norm{f}}\right) & = \int\limits_\R \rho(t)\Omega\left((\norm{f}-t)\frac{f}{\norm{f}}\right)\differential{t}\\
                                                                                        & = \int\limits_{-1}^0 \rho(t)\Omega\left((\norm{f}-t)\frac{f}{\norm{f}}\right)\differential{t}
    \end{align*}
    We check that this function is still non-decreasing along tangential directions, i.e.~we need to show that for $f_T$ s.t. $\semiinner{f_T}{f}=0$ we still have
    \begin{multline}
      \label{eq:integral_estimate}
      \widetilde{\Omega}(f+f_T) = \int\limits_{-1}^0 \rho(t)\Omega\left((\norm{f+f_T}-t)\frac{f+f_T}{\norm{f+f_T}}\right)\differential{t}\\
      \geq \int\limits_{-1}^0 \rho(t)\Omega\left(\left(\norm{f}-t\right)\frac{f}{\norm{f}}\right)\differential{t} = \widetilde{\Omega}(f)
    \end{multline}
    Note that by \cref{lma:circular_bound} we have that $\Omega((\norm{f+f_T}-t)\frac{f+f_T}{\norm{f+f_T}}) \geq \Omega((\norm{f}-t)\frac{f}{\norm{f}})$ for all $t\in[-1,0]$ if $\norm{(\norm{f+f_T}-t)\frac{f+f_T}{\norm{f+f_T}}} \geq \norm{(\norm{f}-t)\frac{f}{\norm{f}}}$ for all $t\in[-1,0]$. But this is clear as it is equivalent to $\modulus{\norm{f+f_T}-t}\geq\modulus{\norm{f}-t}$. As $t$ is non-positive we can drop the modulus to obtain that this happens if $\norm{f+f_T}\geq\norm{f}$ which is just James orthogonality and thus follows from the fact that $\semiinner{f_T}{f}=0$. This proves that the integral estimate \cref{eq:integral_estimate} indeed holds and hence the radially mollified $\widetilde{\Omega}$ is indeed non-decreasing in tangential directions.
  \end{subproof}

  \begin{subproof}[$\Omega$ is as claimed]
    Putting these two observations together we obtain the result. By \cref{sbpf:mollified_regulariser} $\widetilde{\Omega}$ is of the form $\widetilde{\Omega}(f)=h\left(\semiinner{f}{f}\right)$ for some continuous, non-decreasing $h$. But if we consider $\Omega$ along any two distinct, fixed directions given by $f_1,f_2\in\banachspace$, $f_1\neq f_2$, $\norm{f_1}=\norm{f_2}=1$ as $\Omega(t\cdot f_i) = h_{f_i}\left(\semiinner{t\cdot f_i}{t\cdot f_i}\right)$ then the mollifications of both $h_{f_1}$ and $h_{f_2}$ must equal $h$ so $h_{f_1}=h_{f_2}$ almost everywhere. Further by continuity of $h$ they can only differ in points of discontinuity of $h_{f_1}$ and $h_{f_2}$. As each $h_{f_i}$ is a monotone function on the positive real line it can only have countably many points of discontinuity. Clearly as above bounds are only making statements about values outside a given circle and $h$ is itself monotone, each $h_{f_i}$ is free to attain any value within the monotonicity constraint in those points of discontinuity. This shows that $\Omega$ is of the claimed form.
  \end{subproof}
\end{proof}

\begin{remark}
  We see that everything we say about $\Omega$ in this section relies crucially on the observation that it being admissible is a statement about its behaviour along tangents as stated in \cref{lma:admissibility}. But there is in fact no tangent into the complex plane, i.e.\ for fixed $\hat{f}$ there is no tangent that intersects the ray $\{t\cdot e^{i\theta}\cdot\hat{f}\setseparator t\in\R\}$ for any $\theta$. Likewise it is not possible to reach any point along said ray via an ``out and back'' argument as in \cref{spf:half_space_bound} of the proof of \cref{lma:circular_bound}. For this reason it is currently not clear whether one can say anything about the situation in complex vector spaces.
\end{remark}

\section{The solution is determined by the space}
\label{sec:space-dependence}
First of all, while it has been known that for regularisers which are a strictly increasing function of the norm every solution is within the linear span of the data, the proofs in~\cref{sec:representer-theorem} show immediately that something stronger can be said. For a regularised interpolation problem with an admissible regulariser to have a solution which is not in the linear span of the data the regulariser must have a flat region and the solution then has to lie within the flat region.\\
But there is more to be said, in fact it turns out that for admissible regularisers the set of solutions in the linear span is independent of the regulariser.\\
In~\cite{miccelli2004} Micchelli and Pontil consider the minimal norm interpolation problem
$$\inf\{\norm{x}_X \setseparator x\in X, L_i(x) = y_i\,\forall i\in\N_m\}$$
where $X$ is a Banach space and $L_i$ are continuous linear functionals on $X$. Hence this agrees with \cref{eq:regularised_interpolation_semiinner} for $h(t)=\sqrt{t}$ i.e. $\Omega(f)={(\semiinner{x}{x})}^{\frac{1}{2}}$ and $X=\banachspace$ a uniformly convex, uniformly smooth Banach space, giving the minimal norm interpolation problem
\begin{equation}
\label{eq:minimal_norm_interpolation}
  \min\{\norm{f}_\banachspace \setseparator f\in\banachspace, x^\ast_i(f) = \semiinner{f}{x_i} = y_i\,\forall i\in\N_m\}
\end{equation}
This leads to the following result.

\begin{theorem}
\label{thm:space_dependence}
  Let $\Omega$ be admissible. Then any $f_0$ which is such that $f^\ast_0 = \sum\limits_{i=1}^m c_i x^\ast_i$ is a solution of~\cref{eq:regularised_interpolation_semiinner} if and only if it is a solution of~\cref{eq:minimal_norm_interpolation}.
\end{theorem}

The proof of this result relies on the following result which was proved by Micchelli and Pontil in~\cite{miccelli2004}.
\begin{proposition}[Theorem 1 in~\cite{miccelli2004}]
\label{prop:minimal_norm_solution}
$f_0$ is a solution of~\cref{eq:minimal_norm_interpolation} if and only if it satisfies the constraints\\
$x^\ast_i(f_0)=y_i$ and there is a linear combination of the continuous linear functionals defining the problem which peaks at $f_0$, i.e.~there exists $(c_1,\ldots,c_m)\in\R^m$ such that
$$\sum\limits_{i=1}^m c_i x^\ast_i(f_0) = \norm{\sum\limits_{i=1}^m c_ix^\ast_i}_{\banachspace^\ast}\cdot\norm{f_0}_\banachspace$$
\end{proposition}

Using this result it is easy to proof \cref{thm:space_dependence}.\\

\begin{proof}[Of~\cref{thm:space_dependence}]
  \begin{subproof}[A solution of~\cref{eq:regularised_interpolation_semiinner} is a solution of~\cref{eq:minimal_norm_interpolation}]
    Assume that $f_0$ is a solution of~\cref{eq:regularised_interpolation_semiinner} such that $f^\ast_0 = \sum\limits_{i=1}^mc_i x^\ast_i$. Then trivially $f_0$ satisfies the interpolation constraints and by definition
    $$f^\ast_0(f_0) = \semiinner{f_0}{f_0}=\norm{f_0}_\banachspace^2=\norm{f^\ast_0}_{\banachspace^\ast}\cdot\norm{f_0}_\banachspace$$
    so $f^\ast_0$, which is a linear combination of the continuous linear problems defining the problem, peaks at $f_0$. Thus by~\cref{prop:minimal_norm_solution} $f_0$ is a solution of~\cref{eq:minimal_norm_interpolation}.
  \end{subproof}

  \begin{subproof}[A solution of~\cref{eq:minimal_norm_interpolation} is a solution of~\cref{eq:regularised_interpolation_semiinner}]
    Assume $f_0$ is a solution of~\cref{eq:minimal_norm_interpolation}. Then by~\cref{prop:minimal_norm_solution} there exists\\
    $(c_1,\ldots,c_m)\in\R^m$ such that the functional $\sum\limits_{i=1}^m c_i x^\ast_i$ peaks at $f_0$, i.e.
    $$\sum\limits_{i=1}^m c_i x^\ast_i(f_0) = \norm{\sum\limits_{i=1}^m c_i x^\ast_i}_{\banachspace^\ast}\cdot\norm{f_0}_\banachspace$$
    But then for any $g\in Z=\{f\in\banachspace \setseparator x^\ast_i(f)=\semiinner{f}{x_i}=0,\,\forall i=1,\ldots,m\}$ we have that
    $$\norm{\sum\limits_{i=1}^m c_i x^\ast_i}_{\banachspace^\ast}\cdot\norm{f_0}_\banachspace = \sum\limits_{i=1}^m c_i x^\ast_i(f_0) = \sum\limits_{i=1}^m c_i x^\ast_i(f_0 + g) < \norm{\sum\limits_{i=1}^m c_i x^\ast_i}_{\banachspace^\ast}\cdot\norm{f_0+g}_\banachspace$$
    where the last inequality is strict because $\sum\limits_{i=1}^m c_i x^\ast_i$ peaks at $f_0$ and by strict convexity it peaks at a unique point. But this inequality shows that
    $$\norm{f_0}_\banachspace < \norm{f_0+g}_\banachspace$$
    for all $g\in Z$ and thus as $\Omega$ is admissible also
    $$\Omega(f_0) < \Omega(f_0+g)$$
    and $f_0$ is a solution of~\cref{eq:regularised_interpolation_semiinner}.
  \end{subproof}
\end{proof}

This result shows that any admissible regulariser on a uniformly convex and uniformly smooth Banach space has a unique solution in the linear span of the data and the solution is the same for every admissible regulariser. This in particular means that it is the choice of the function space, and only the choice of the space, which determines the solution of the problem. We are thus free to work with whichever regulariser is most convenient in application. Computationally in many cases this is likely going to be $\frac{1}{2}\norm{\cdot}^2$, for theoretical results other regularisers may be more suitable, such as in the afore mentioned paper~\cite{miccelli2004} which heavily relies on a duality between the norm of the space and its continuous linear functionals.
\begin{appendix}
\bibliographystyle{acm}
\bibliography{bibliography.bib}
\section{Appendix}

\begin{proof}[Of~\cref{prop:dual_map_cts}]
    We begin by showing norm-to-weak continuity and subsequently extend it to norm-to-norm continuity.\\
    Since $\banachspace$ is reflexive the weak and weak$\ast$ topologies on $\banachspace^\ast$ coincide, so we need to show that if $x_n\rightarrow x$ in norm then $x^\ast_n(y)\rightarrow x^\ast(y)$ for all $y\in\banachspace$.\\
    Now as $\norm{x^\ast_n}_{\banachspace^\ast}=\norm{x_n}_\banachspace$ the sequence $(x^\ast_n)$ is bounded so it has a weakly$\ast$ convergent subsequence $x^\ast_{n_k}\overset{\ast}{\rightharpoonup}\overline{x}^\ast$. By~\cite{brezis2011} proposition 3.13 (iv) we then have
    $$x^\ast_{n_k}\left(x_{n_k}\right)\converges{k}\overline{x}^\ast(x)$$
    But $x^\ast_{n_k}(x_{n_k})=\norm{x_{n_k}}^2\rightarrow\norm{x}^2$ and so $\overline{x}^\ast(x)=\norm{x}^2$. By~\cite{brezis2011} proposition 3.13 (iii) we further know that $\norm{\overline{x}^\ast}\leq\liminf\norm{x^\ast_{n_k}}=\norm{x}$. By strict convexity there is a unique element with those two properties and hence $\overline{x}^\ast=x^\ast$.\\
    Note that this means that for any subsequence there exists a further subsequence converging to a unique limit. This means that in fact the entire sequence converges to this unique limit. Hence indeed $x^\ast_n\rightharpoonup x^\ast$ as claimed.\\
    \ \\
    Having established norm-to-weak continuity one can easily extend it to norm-to-norm continuity using \cite{brezis2011} proposition 3.32. Since $\limsup\norm{x^\ast_n}_{\banachspace^\ast}=\norm{x}_\banachspace=\norm{x^\ast}_{\banachspace^\ast}$ all the assumptions of proposition 3.32 in \cite{brezis2011} are satisfied and so indeed $x^\ast_n\rightarrow x^\ast$ in norm.
\end{proof}

\end{appendix}

\end{document}